\documentclass[10pt,twocolumn,letterpaper]{article}

\usepackage[pagenumbers]{cvpr} 
\definecolor{cvprblue}{rgb}{0.21,0.49,0.74}
\usepackage[pagebackref,breaklinks,colorlinks,allcolors=cvprblue]{hyperref}
\usepackage{booktabs}
\usepackage{booktabs}    
\usepackage{color}       
\usepackage{soul}        
\usepackage{graphicx} 
\usepackage{multirow}
\usepackage{multicol}
\usepackage{amsthm}

\usepackage{colortbl}
\newtheorem{theorem}{Theorem}
\usepackage{algorithm}
\usepackage{algorithmicx}
\usepackage{algpseudocode}
\usepackage{amsmath}
\def\paperID{} 
\def\confName{}
\def\confYear{}

\title{V-ITI: Mitigating Hallucinations in Multimodal Large Language Models via Visual Inference-Time Intervention}

\usepackage{hyperref}
\usepackage{footmisc} 
\renewcommand{\thefootnote}{%
  \ifcase\value{footnote}\or *\or \dag\or \ddag\fi%
}

\author{
  \begin{tabular}{c}
    Nan Sun\textsuperscript{1,2\hyperref[fn:star]{*}\hyperref[fn:dag]{\dag}}, Zhenyu Zhang\textsuperscript{3\hyperref[fn:star]{*}}, Xixun Lin\textsuperscript{1,2\hyperref[fn:ddag]{\ddag}}, Kun Wang\textsuperscript{4}, Yanmin Shang\textsuperscript{1,2}
  \end{tabular}\\
  \begin{tabular}{c}
    Naibin Gu\textsuperscript{1,2}, Shuohuan Wang\textsuperscript{3}, Yu Sun\textsuperscript{3}, Hua Wu\textsuperscript{3}, Haifeng Wang\textsuperscript{3}, Yanan Cao\textsuperscript{1,2}
  \end{tabular}
  \\[5pt] 
  \begin{tabular}[t]{c} 
    \textsuperscript{1} Institute of Information Engineering, Chinese Academy of Sciences, Beijing, China \\
    \textsuperscript{2} School of Cyber Security, University of Chinese Academy of Sciences, Beijing, China \\
    \textsuperscript{3} Baidu Inc., Beijing, China \quad
    \textsuperscript{4} Nanyang Technological University\\
    \texttt{\small \{sunnan, linxixun\}@iie.ac.cn}\\
    \texttt{\small \{zhangzhenyu07,wangshuohuan\}@baidu.com}
  \end{tabular}
}

\begin{document}
\maketitle
\footnotetext[1]{\label{fn:star} Equal contribution.} 
\footnotetext[2]{\label{fn:dag} This work was done during an internship at Baidu Inc.} 
\footnotetext[3]{\label{fn:ddag} Corresponding author: Xixun Lin.} 

\begin{abstract}
Multimodal Large Language Models (MLLMs) excel in numerous vision-language tasks yet suffer from hallucinations, producing content inconsistent with input visuals, that undermine reliability in precision-sensitive domains. 
This issue stems from a fundamental problem of \textbf{visual neglect}, where models fail to adequately prioritize input images. 
Existing methods typically alleviate hallucinations by intervening in the attention score or output logits, focusing on ``how to intervene'' but overlooking the prerequisite ``when to intervene'', which leads to the \textbf{``over-intervention''} problem and subsequently introduces new hallucinations and unnecessary computational overhead. 
To address this gap, we first investigate the mechanism of visual neglect and reveal it can be accurately detected via head-level activation patterns in MLLMs. 
We thus propose V-ITI, a lightweight visual inference-time intervention framework integrating a Visual Neglect Detector that identifies visual neglect via head-level discriminative probes and a Visual Recall Intervenor that modulates activations with prestored visual activation information only when the visual neglect is detected. Extensive experiments across eight benchmarks and different MLLM families demonstrate that V-ITI consistently mitigates vision-related hallucinations while preserving general task performance.
\end{abstract}    
\begin{figure}[t]
\begin{center}
\includegraphics[width=\linewidth]{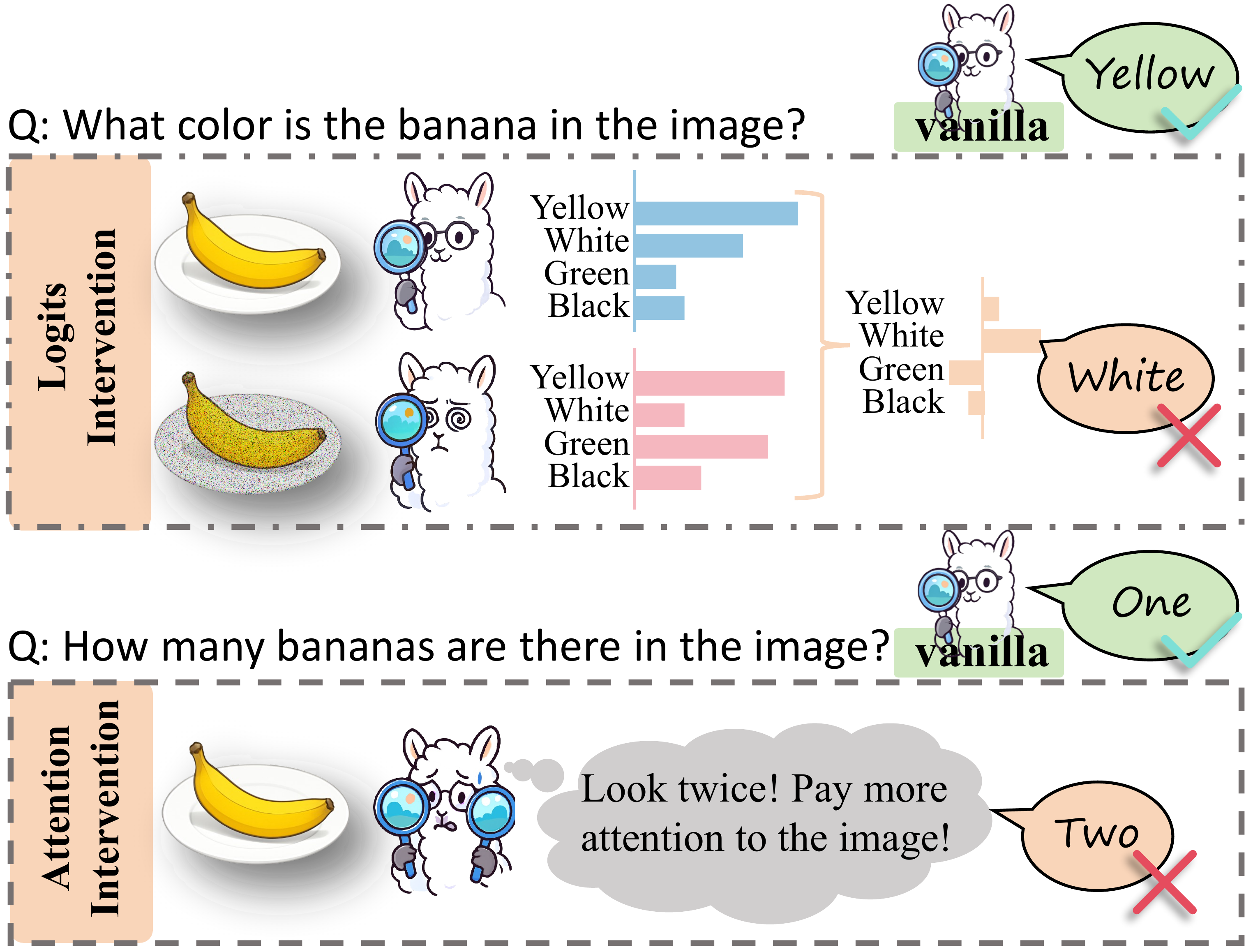}
\end{center}
\caption{Illustration of the phenomenon of ``over-intervention''. For logits intervention (upper), contrasting logits from perturbed and undisturbed visual inputs suppresses the correct answer’s logits, which forces error of answering \textit{``White''} instead of the correct color \textit{``Yellow''}. For attention intervention (lower), overly enhancing visual attention induces new quantity hallucinations, where the model initially correct in answering \textit{``One''}, mistakenly answers \textit{``Two''} due to repeated due to excessive attention.}
\label{fig1} 
\vspace{-15pt}
\end{figure}
\section{Introduction}
Multimodal Large Language Models (MLLMs)~\cite{zhang2024mm} have demonstrated impressive performance across a variety of vision-language tasks~\cite{zhou2023vision+}, such as image captioning and visual question answering. However, their reliability remains a major concern due to hallucinations~\cite{huang2024visual}---a critical flaw where models generate content inconsistent with input images, for example, fabricating non-existent objects or making factually incorrect inferences about visual details. Such hallucination issues are particularly unacceptable in safety-sensitive fields, such as biometrics~\cite{tiong2024flexible}, autonomous driving~\cite{cui2024survey}, and medical diagnosis~\cite{wu2023medklip}.
Recent studies~\cite{liu2024paying,zhang2025mllms,yin2025clearsight} have traced the root of hallucination to a phenomenon termed \textbf{\textit{Visual Neglect}}.
In response to this issue, some approaches adopt resource-intensive pipelines from fine-tuning (FT)~\cite{gunjal2024detecting}, reinforcement learning (RL)~\cite{yu2024rlhf} and retrieval-augmented generation (RAG)~\cite{qu2025alleviating}, which are difficult to replicate. This paper, therefore, focuses solely on lightweight interventions at the inference stage to re-emphasize the crucial role of visual components.

Mainstream intervention approaches can be roughly categorized into two types:
(1) \textbf{\textit{Logits Intervention}}, which modulates the logits for next-token prediction in a contrastive manner using perturbed visual inputs. For example, Visual Contrastive Decoding (VCD)~\cite{vcd2024cvpr} replaces original images with noised ones to obtain negative logits, then subtracts them from the original logits to enhance visual grounding during decoding. Instruction Contrastive Decoding (ICD) inherits this concept but applies the perturbation in the instruction part instead of the image.
(2) \textbf{\textit{Attention Intervention}}, which sharpens the model’s attention toward visual tokens. For instance, OPERA~\cite{zoulook} finds that during hallucination, attention shifts to 'anchor tokens' and applies a penalty to refocus attention on the visual tokens. INTER~\cite{dong2025inter} further improves OPERA's sampling by reapplying MLLMs' understanding of multimodal interaction information to force the model re-focus to the visual part.

However, most existing methods primarily fixate on the question of \textit{\textbf{``How to intervene''}} and enforce uniform intervention across all tokens and samples, yet overlook a fundamental prerequisite: \textit{\textbf{``When to intervene''}}. 
This omission results in unnecessary computational overhead and increases the risk of \textbf{\textit{“over-intervention”}}, that is, applying unnecessary interventions even when visual neglect is absent. Such indiscriminate intervention can introduce new hallucinations despite initially correct model predictions.
Take Fig.~\ref{fig1} as an example, for questions the model originally answers correctly, logits intervention may unnecessarily suppress the correct answer’s logits when they align with those from perturbed visual inputs, meanwhile, attention intervention might overly amplify visual focus, causing issues such as repeated counting. Consequently, many existing hallucination mitigation methods inadvertently degrade the general-task performance of MLLMs.

To explore the causes of ``over-intervention'' and bridge this gap, we begin by examining the core mechanisms behind hallucinations in MLLMs. 
As perturbation steps increase (Fig.~\ref{fig2}), both F1 score and accuracy consistently decline. Meanwhile, the attention allocated to visual tokens overall diminishes, and within the subset of visual tokens, the model shifts from focusing on question-relevant regions to failing to capture valid information. These trends reflect the onset of visual neglect, ultimately leading to degraded performance.
These observations suggest that attention patterns can serve as a signal to detect visual neglect, offering a natural way to determine when intervention is necessary.

\begin{figure}[t]
\begin{center}
\includegraphics[width=\linewidth]{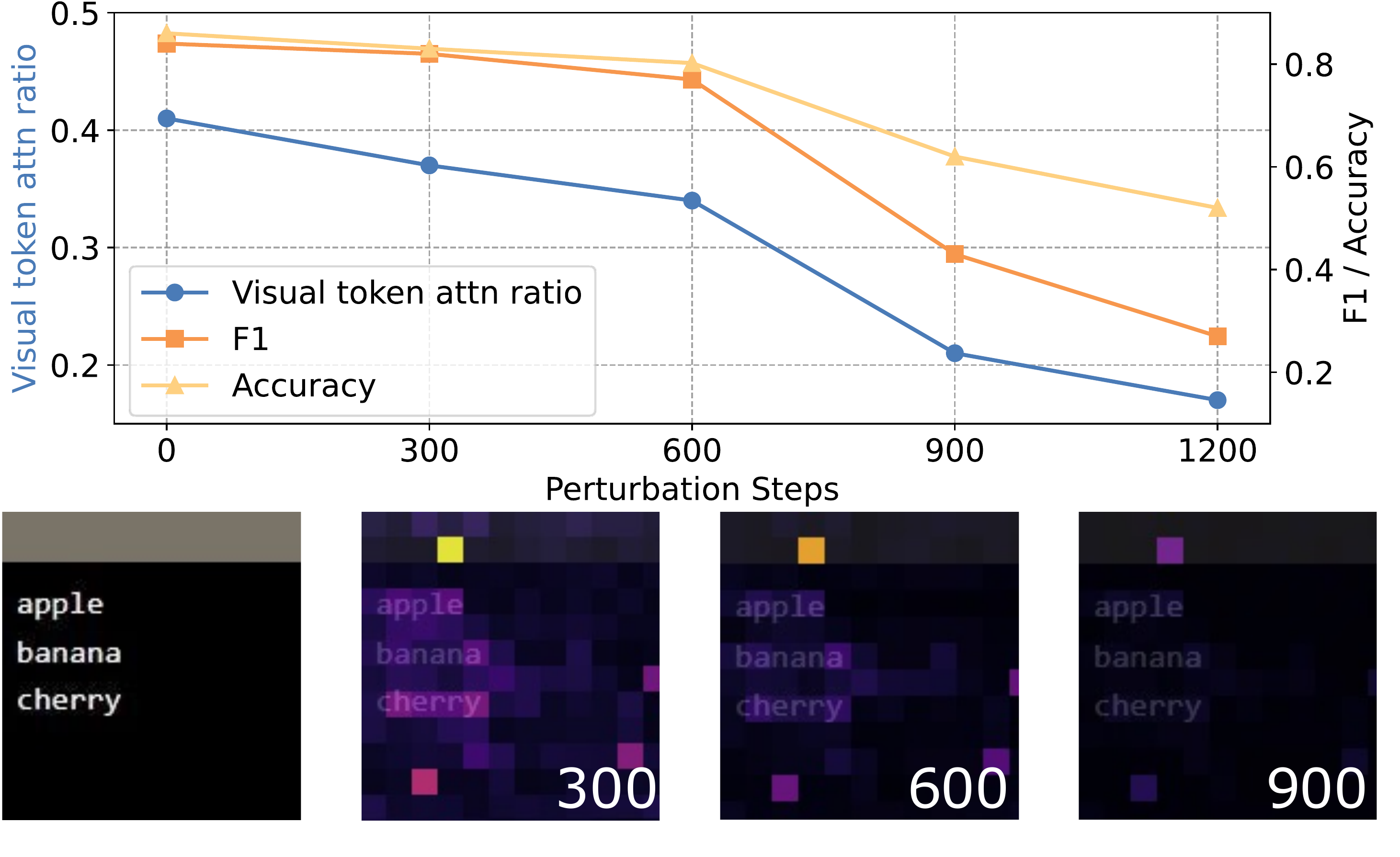}
\end{center}
\caption{As Gaussian perturbations increase, the model’s attention to visual tokens diminishes, and the attention heatmap reveals a loss of focus on question-relevant regions. This visual neglect weakens the model’s ability to perceive visual evidence, ultimately leading to a decline in accuracy and F1 score on the POPE dataset.}
\label{fig2}
\vspace{-15pt} 
\end{figure}
Building on this insight, we propose V-ITI, a lightweight inference-time intervention framework designed to mitigate hallucinations in MLLMs by addressing the two key questions: \textbf{\textit{“When and how to intervene?”}}. Specifically, V-ITI consists of two pivotal modules:
(1) \textbf{\textit{Visual Neglect Detector}}, which determines when to intervene. Here we train a set of head-level probes using paired perturbed and clean images to detect visual neglect through activation discrepancies. During inference, these probes identify neglect instances via the alignment between probe directions and head level activations and gate the subsequent module.
(2) \textbf{\textit{Visual Recall Intervenor}}, which decides how to intervene. It integrates current head activations with retained visual attention weights to reinforce attention toward relevant visual tokens, ensuring the model maintains focus on informative visual regions.
Extensive experiments demonstrate the effectiveness of V-ITI, showing consistent improvements in both hallucination mitigation and general multimodal capabilities across eight public benchmarks and different MLLM families.
Overall, our contributions are threefold: 
\begin{itemize}
    \item \textbf{Problem Disclose.} We uncover an overlooked issue in hallucination mitigation for MLLMs, whether visual neglect truly occurs, and provide a systematic analysis of its underlying causes to guide targeted intervention design.
    \item \textbf{Intervention Paradigm.} We introduce V-ITI, the first visual inference-time intervention framework that systematically addresses both when and how to intervene. By integrating a Visual Neglect Detector and a Visual Recall Intervenor, it enables selective and adaptive intervention based on detected visual neglect.
    \item \textbf{Empirical Evaluation.} Extensive experiments across different MLLM families and benchmarks demonstrate that V-ITI effectively mitigates hallucinations while maintaining strong general multimodal capabilities.
\end{itemize}

\section{Related Works}
\noindent\textbf{\textit{Multimodal Large Language Models (MLLMs).}} Early MLLMs employed BERT-analogous models for basic cross-modal fusion, where the feature spaces of different modalities were distinct and only loosely aligned. With the development of LLMs, the integration of open-source models like the LLaMA family enabled greater adaptability, fostering more robust interactions between modalities and improving MLLMs' capabilities. Notable models reflecting this progression include LLaVA~\cite{liu2024visual}, Qwen-VL~\cite{bai2023qwen}, GLM4V~\cite{wang2023cogvlm}, and LLaVA-Next~\cite{liu2024llavanext}. Contemporary MLLMs are typically trained in two stages: i) cross-modal pretraining for feature alignment and ii) instruction tuning for task-specific adaptability. While these models perform well, hallucination remains a significant challenge, limiting their reliability in real-world applications.

\noindent\textbf{\textit{Hallucination Mitigation in MLLMs.}} 
Early attempts in mitigating hallucinations in MLLMs centered on refining modality alignment \cite{rohrbach2018object} and curbing biases \cite{kim2023exposing}, with recent advances spanning FT~\cite{gunjal2024detecting}, RL~\cite{yu2024rlhf} and RAG~\cite{yu2024rlhf}. However, most of these methods incur substantial resource costs. Two high-performance inference-time paradigms have emerged. The first is \textbf{\textit{Logits Intervention}}, which adjusts output logits distributions represented by contrastive decoding (CD) \cite{chuang2023dola}, VCD \cite{vcd2024cvpr}, and ICD \cite{wang2024mitigating}. This paradigm often fails to guarantee consistent performance due to unintended noise introduced during decoding. The second is \textbf{\textit{Attention Intervention}}, which modulates attention to enhance visual relevance. A representative lightweight method is OPERA \cite{opera2024cvpr}, penalizing incorrect attention aggregation in beam search. INTER \cite{dong2025inter} further refines its sampling by reusing MLLMs’ multimodal interaction comprehension.
\section{Preliminary}\label{sec:preliminary}
Given a MLLM $\mathcal{M}_{\boldsymbol{\Phi}}$ with parameters \( \boldsymbol{\Phi} \), we adopt the standard input template \textsf{[}$\langle$\textbf{text-system}$\rangle$ $\langle$\textbf{vision-image}$\rangle$ $\langle$\textbf{text-instruction}$\rangle$\textsf{]} for processing queries, where both textual and visual tokens are projected into a unified feature space of dimension \( HD \). Here, \( H \) is the number of attention heads, \( D \) represents the dimensionality of each head. Consequently, the input embedding sequence of length \( n \) at the \( l \)-th layer is represented as $\boldsymbol{X}_{l} = [\boldsymbol{x}_l^1, \dots, \boldsymbol{x}_l^{v_s}, \dots, \boldsymbol{x}_l^{v_e}, \dots, \boldsymbol{x}_l^n] \in \mathbb{R}^{n \times HD}$
where the embedding \( \boldsymbol{x}_l^{v_s:v_e}=\{ \boldsymbol{x}_l^i \}_{i=v_s}^{v_e} \) correspond to the vision tokens, while the rest of the tokens represent the textual components. The index \( l \) ranges from 0 to \( L-1 \), where \( L \) denotes the total number of transformer layers. Specifically, \( \boldsymbol{X}_0 \) corresponds to the initial embeddings.
Taking multi-head attention (MHA) as an example, the input embeddings \( \boldsymbol{X}_{l} \) at the \( l \)-th layer are projected into query (\(\boldsymbol{Q}^h \)), key (\(\boldsymbol{K}^h \)), and value (\(\boldsymbol{V}^h \)) matrices. The \( h \)-th attention head performs the following operation:
\begin{equation}
\small
    \boldsymbol{Q}^h = \boldsymbol{X}_{l} \boldsymbol{W}_Q^h, \quad \boldsymbol{K}^h = \boldsymbol{X}_{l} \boldsymbol{W}_K^h, \quad \boldsymbol{V}^h = \boldsymbol{X}_{l} \boldsymbol{W}_V^h,
\end{equation}
where \( \boldsymbol{W}_Q^h, \boldsymbol{W}_K^h, \boldsymbol{W}_V^h \in \mathbb{R}^{DH \times D} \) are learned weight matrices for the \( h \)-th attention head, \( \boldsymbol{Q}^h, \boldsymbol{K}^h, \boldsymbol{V}^h \) are in \( \mathbb{R}^{n \times D} \), and the layer index \( l \) is omitted here for convenience.
In that case, the attention scores \( \boldsymbol{A}_l^h \in \mathbb{R}^{n \times n} \) are computed as
\begin{equation}
\small
\boldsymbol{A}_l^h = \text{softmax} \left( \boldsymbol{Q}^h (\boldsymbol{K}^h)^\top/\sqrt{D} \right). 
\end{equation}
The head-wise activation \( \boldsymbol{o}_l^h \in \mathbb{R}^{n \times D} \) of the \( h \)-th head for layer \( l \) is then computed as
\begin{equation}
\small
\label{eq:3}
\boldsymbol{o}_l^h = \boldsymbol{A}_l^h \boldsymbol{V}^h,
\end{equation}
where \( \boldsymbol{V}^h \in \mathbb{R}^{n \times D} \) is the value matrix for the \( h \)-th head, and \( \boldsymbol{o}_l^h \in \mathbb{R}^{n \times D} \) is the output of the \( h \)-th attention head for layer \( l \).  
After processing all attention heads, the outputs are concatenated and projected back to the original space
$\boldsymbol{O}_l = \text{concat}(\boldsymbol{o}_l^1, \dots, \boldsymbol{o}_l^H) \boldsymbol{W}_O$.
After the attention mechanism, the output \( \boldsymbol{o}_l \in \mathbb{R}^{n \times DH} \) is passed through a position-wise feed-forward network (FFN) consisting of two linear layers
$\operatorname{FFN}(\boldsymbol{O}_l) = \phi \left( \boldsymbol{O}_l \boldsymbol{W}_1 \right) \boldsymbol{W}_2^{\top}$,
where \( \phi \) is an activation function like SiLU~\cite{liu2020evolving}, and \( \boldsymbol{W}_1, \boldsymbol{W}_2 \in \mathbb{R}^{DH \times mDH} \) are the weight matrices, with \( m = 4 \) typically. The output of the FFN is then passed through layer normalization and residual connections, resulting in \( \boldsymbol{X}_{l+1} \), which serves as the input to the next layer.
After all the layers, we use the output of the last token for prediction. The answer \(Y = \left[ y^1, y^2, \dots, y^T \right]\) is sampled auto-regressively and the \(t\)-th
token follows the probability distribution \( f_{\boldsymbol{\Phi}} \) conditioned on the query:
\begin{equation}
\small
y^t \sim f_{\boldsymbol{\Phi}} (y^t \mid \boldsymbol{X}_0, y^{1:t-1}))\propto \exp \left( \text{logit}_{\boldsymbol{\Phi}} (y^t \mid \boldsymbol{X}_0, y^{1:t-1}) \right).
\end{equation}
\section{Method}\label{sec:method}
Fig.~\ref{fig3} presents the overall architecture of V-ITI, which integrates a Visual Neglect Detector (Sec.~\ref{sec:VND}) and a Visual Recall Intervenor (Sec.~\ref{sec:VRI}) to effectively mitigate hallucinations while avoiding over-intervention. Furthermore, we provide a theoretical analysis (Sec.~\ref{sec:TA}) of V-ITI to better understand its underlying mechanisms and effectiveness.
\begin{figure*}[h]
\begin{center}
\includegraphics[width=0.9\linewidth]{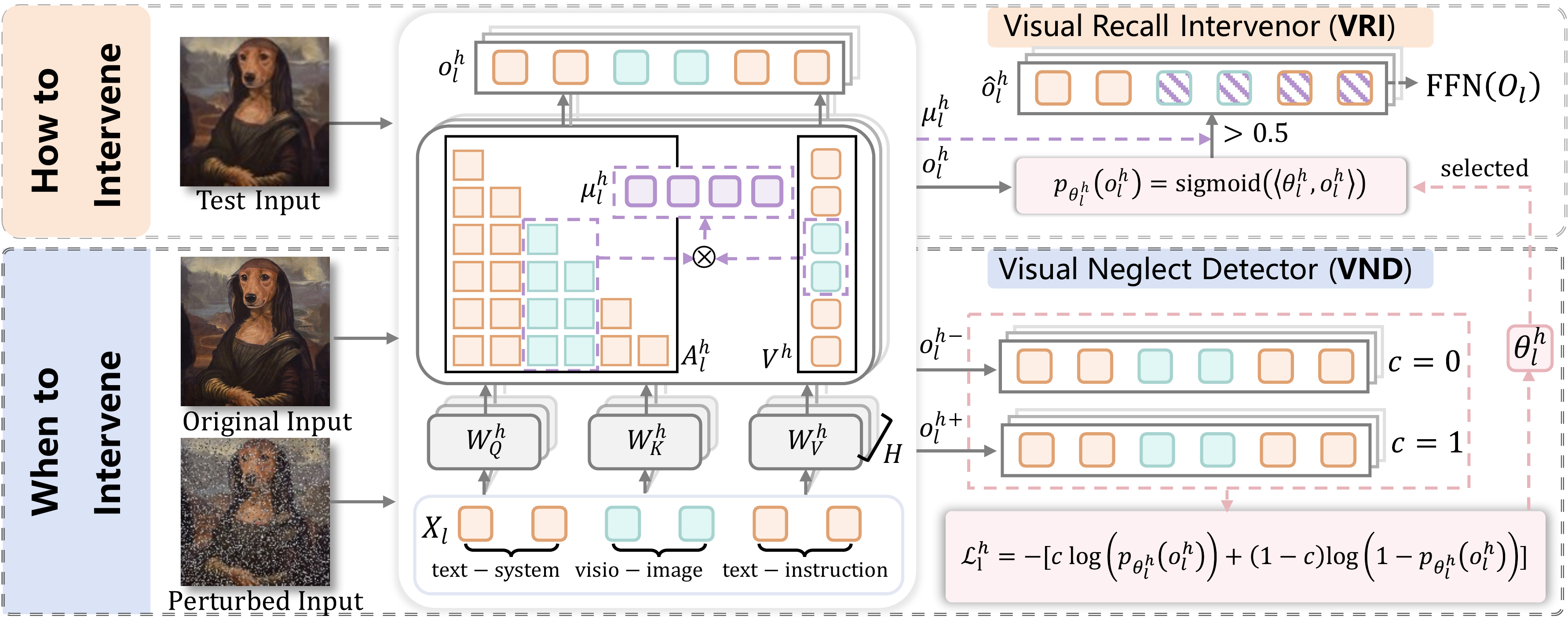}
\end{center} 
\caption{Illustration of the overall V-ITI architecture. To avoid over-intervention in hallucination mitigation, we propose two modules. The Visual Neglect Detector (VND) determines when intervention is needed by discriminating head-level activation patterns, while the Visual Recall Intervenor (VRI) addresses how to intervene by integrating the original head output with retained visual activations.}
\label{fig3}
\end{figure*}
\subsection{Visual Neglect Detector}\label{sec:VND}
Recent works on inference-time intervention in the NLP field have shown that head-level probes in LLMs can identify directions in the activation space that provide interpretable evidence of whether the model is ``telling the truth''~\cite{li2023inference}.
Building on this insight, we focus on identifying head-level activation directions indicative of visual neglect in multimodal scenario. To this end, we propose a \textbf{Visual Neglect Detector (VND)}, which trains head-level probes to pinpoint activation patterns associated with visual neglect. By analyzing the alignment between probe directions and model activations, the detector identifies when visual neglect occurs and determines  ``when to intervene''.

\noindent\textbf{Setup.} To enable the probes to detect visual neglect, we designed two types of perturbations to construct positive samples. The first perturbation simulates a scenario where the model is unsure of where to focus by adding Gaussian noise for 1000 steps. The second perturbation simulates a situation where the model excessively focuses on irrelevant areas by replacing the embeddings corresponding to the top 75\% of attention with smoothed embeddings from the remaining 25\%. All of these operations are based on data from the instruction tuning stage of LLaVA's training, which are considered as negative data and includes 150K GPT-generated samples and 515K VQA data from academic tasks. 

\noindent\textbf{Detector Architecture.} VND is essentially a reformulation of MHA. Let each labeled sample be denoted as \( (\boldsymbol{X}, c) \), where \( c \in \{0, 1\} \) indicates the presence (1) or absence (0) of visual neglect. For each sample, we obtain the corresponding head-wise activations \( \boldsymbol{o}_l^h \) from the MLLM as defined in Eq.~\ref{eq:3}. Specifically, our detector consists of a set of linear probes at each layer, denoted by \( \boldsymbol{\theta} = \{\theta_{l}^h : h = 1, \dots, H\}_{l=0}^{L-1} \), where \( \theta_{l}^h \in \mathbb{R}^{D} \). These probes are then dot-multiplied with the corresponding head-wise activations \( \boldsymbol{o}_l^h \) to produce a binary classification result. We define the output at the \( l \)-th layer and \( h \)-th head as:
\begin{equation}
\small
p_{\theta_l^h}(\boldsymbol{o}_l^h) = \text{sigmoid}\left( \langle \theta_l^h, \boldsymbol{o}_l^h \rangle \right).
\end{equation}
For each labeled sample \( (\boldsymbol{X}, c) \), the binary cross-entropy loss \( \mathcal{L}_l^h \) is given by:
\begin{equation}
\small
\mathcal{L}_l^h = - \left[ c \log(p_{\theta_l^h}(\boldsymbol{o}_l^h)) + (1 - c) \log(1 - p_{\theta_l^h}(\boldsymbol{o}_l^h)) \right].
\end{equation}
The total loss for the probe \( \theta_{l}^h \) is computed by averaging the loss across the entire training dataset. Taking the training results in Fig.~\ref{fig4} on LLaVA-1.5 for example, we sort the accuracy of the probes and observe two key findings: First, all head probes achieve accuracy greater than 50\% with peak detection performance reaching 87.6\%, demonstrating that visual neglect can be effectively detected. Second, the distribution of high-performing probes is sparse and primarily concentrated in the middle layers. This insight informs our approach to adaptively select probes in a sparse manner, prioritizing those with higher accuracy. Similar observations have been made in other MLLMs as well. These findings provide valuable guidance on how to effectively leverage accuracy for global selection of when to intervene.

\begin{figure}[t]
\begin{center}
\includegraphics[width=0.85\linewidth]{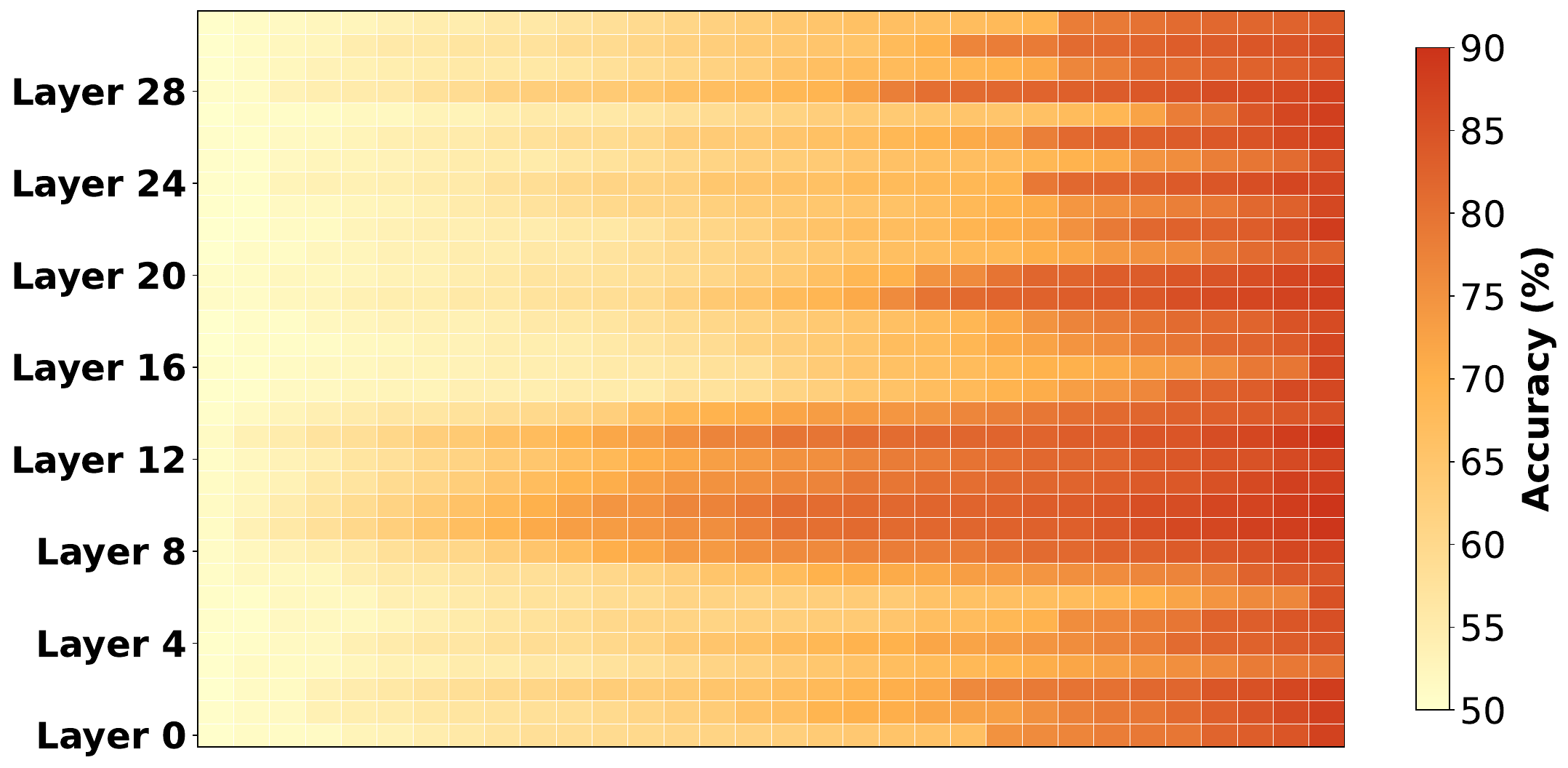}
\end{center} 
\caption{Sorted probe accuracies on validation set of all attention heads across all layers in the LLaVA-1.5 model.}
\label{fig4}
\end{figure}

\subsection{Visual Recall Intervenor}\label{sec:VRI}
When the head exhibits visual neglect, the current activation states require supplementation with visual information to mitigate hallucinations. To address this need, our \textbf{Visual Recall Intervenor(VRI)} introduce an additional visual activation component \( \boldsymbol{\mu}_l^h \) during MHA computation and store this component for subsequent intervention. Specifically,
the visual activation \( \boldsymbol{\mu}_l^h \) is computed by restricting attention to vision tokens in range \([v_s, v_e]\), with re-normalized weights to ensure valid summation as
\begin{equation}
\small
\label{eq:visual_activation}
\boldsymbol{\mu}_l^h = \left( \frac{\boldsymbol{A}_l^h[:, v_s:v_e]}{\sum_{j=v_s}^{v_e} \boldsymbol{A}_l^h[:, j] + \epsilon} \right) \cdot \boldsymbol{V}^h[v_s:v_e],
\end{equation}
where \( \boldsymbol{A}_l^h[:, v_s:v_e] \) and \( \boldsymbol{V}^h[v_s:v_e] \) denote the submatrices of \( \boldsymbol{A}_l^h \) and \( \boldsymbol{V}^h \) corresponding to vision tokens, and \( \epsilon \) is a small constant to avoid division by zero.

We model visual recall intervention process as a gating mechanism, where the detector probe probability \( p_{\theta_l^h}(\boldsymbol{o}_l^h) \in [0,1] \) serves as the trigger signal. The final modulated head-wise activation \( \boldsymbol{\hat{o}}_l^h \) is thus computed as a weighted combination of the original activation \( \boldsymbol{o}_l^h \) and the visual activation \( \boldsymbol{\mu}_l^h \):
\begin{equation}
\small
\label{eq:gated_intervention}
\boldsymbol{\hat{o}}_l^h =\begin{cases}(1 - \alpha) \cdot \boldsymbol{o}_l^h + \alpha \cdot \boldsymbol{\mu}_l^h, & p_{\theta_l^h}(\boldsymbol{o}_l^h) > 0.5, \\ \boldsymbol{o}_l^h, & p_{\theta_l^h}(\boldsymbol{o}_l^h) \leq 0.5,
\end{cases}\end{equation}
where $\alpha$ is the weight parameter.
To fully leverage the continuous confidence \( p_{\theta_l^h}(\boldsymbol{o}_l^h) \) as an indicator of the severity of visual neglect, we design the weight parameter as 
\begin{equation}
\small
\alpha = \alpha_0 \cdot p_{\theta_l^h}(\boldsymbol{o}_l^h),
\end{equation}
where \( \alpha_0  \in (0,1] \) denotes the intervention strength.
Given that VND has identified the most sensitive attention heads which are sparsely distributed, we select the top-\( \beta \) accuracy heads to control VRI's intervention, where \( \beta \in [0,1] \) and \( \beta = 1 \) corresponds to full intervention across all heads. Intervention is thus restricted to these high-risk visual neglect regions to mitigate hallucinations. When the associated detector probe identifies visual neglect during the current inference step, the pre-stored visual activation \( \boldsymbol{\mu}_l^h \) is employed to modulate the model's activation states.

\begin{algorithm}[t]
\small
\caption{V-ITI with Global Top-\( \beta \) Head Selection}
\label{alg:VRN}
\textbf{Require:} MLLM $\mathcal{M}_{\boldsymbol{\Phi}}$, input at $0$-th layer $\boldsymbol{X}_{0} = [\boldsymbol{x}_0^1, \dots, \boldsymbol{x}_0^{v_s}, \dots, \boldsymbol{x}_0^{v_e}, \dots, \boldsymbol{x}_0^n]$, probes $\{\theta_l^h\}$, intervention strength $\alpha_0$, head selection parameter $\beta \in [0,1]$. \\
\textbf{Output:} Model response $Y = [y^1, y^2, \dots, y^T]$.\\
\textbf{Layer-wise Processing (for $l = 0$ to $L-1$):}
\begin{enumerate}
    \item Compute $\boldsymbol{o}_l^h$ (via Eq.~\eqref{eq:3}) and $\boldsymbol{\mu}_l^h$ (via Eq.~\eqref{eq:visual_activation}).
    \item For each head \( h \) in the top-\( \beta \) selected heads:
    \begin{enumerate}
        \item If \( p_{\theta_l^h}(\boldsymbol{o}_l^h) > 0.5 \): \\
            $\boldsymbol{\hat{o}}_l^h = (1 - \alpha) \cdot \boldsymbol{o}_l^h + \alpha \cdot \boldsymbol{\mu}_l^h$.
        \item Else: $\boldsymbol{\hat{o}}_l^h = \boldsymbol{o}_l^h$.
    \end{enumerate}
    \item For other heads not in the top-\( \beta \) selection, set $\boldsymbol{\hat{o}}_l^h = \boldsymbol{o}_l^h$.
    \item Layer output: $\boldsymbol{O}_l = \text{concat}(\boldsymbol{\hat{o}}_l^1, \dots, \boldsymbol{\hat{o}}_l^H) \boldsymbol{W}_O$.
    \item For the next layer: $\boldsymbol{X}_{l+1} = \text{FFN}(\boldsymbol{O}_l) + \boldsymbol{X}_l$.
\end{enumerate}
\textbf{Decoding:} Sample $Y$ auto-regressively from final layers.
\end{algorithm}

\subsection{Theoretical Analysis}\label{sec:TA}
\noindent\textbf{Theoretical Guarantees.}
To theoretically verify that V-ITI strengthens the correlation between model activations and visual tokens, we analyze the Mutual Information (MI) between head-wise activations and visual token subsets before and after intervention.
\begin{theorem}
Let \( \boldsymbol{o}_l^h \) denote the original head-wise activation and \( \boldsymbol{\hat{o}}_l^h \) denote the modulated activation after intervention. Let \( \boldsymbol{X}[v_s:v_e] \subseteq \boldsymbol{X}_l \) be the visual subset of input tokens at $l$-th level. The MI between them satisfies
\begin{equation}\label{theorem:1}
\small
I\left( \boldsymbol{\hat{o}}_l^h; \boldsymbol{X}[v_s:v_e] \right) \geq I\left( \boldsymbol{o}_l^h; \boldsymbol{X}[v_s:v_e] \right).
\end{equation}
\end{theorem}
\begin{proof}
Using the chain rule of mutual information, we substitute modulated activation $\boldsymbol{\hat{o}}_l^h$ via Eq.~\eqref{eq:gated_intervention} and rearrange the terms to obtain
\begin{equation}
\small
\begin{split} 
I\left( \boldsymbol{\hat{o}}_l^h; \boldsymbol{X}[v_s:v_e] \right) = I\left( (1-\alpha)\boldsymbol{o}_l^h + \alpha\boldsymbol{\mu}_l^h; \boldsymbol{X}[v_s:v_e] \right)\\
= I\left( \boldsymbol{o}_l^h + \alpha\left(\boldsymbol{\mu}_l^h-\boldsymbol{o}_l^h\right); \boldsymbol{X}[v_s:v_e] \right).
\end{split}
\end{equation}
Comparing with the original MI of $\boldsymbol{o}_l^h$, we apply the additivity property of mutual information for any random variables $a, b, z$: $I(a+b; z) = I(a; z) + I(b; x \mid a)$. We let $a = \boldsymbol{o}_l^h$, $b = \alpha(\boldsymbol{\mu}_l^h - \boldsymbol{o}_l^h)$, and $x = \boldsymbol{X}[v_s:v_e]$. The theorem is thus equivalent to proving
\begin{equation}\label{miuo}
\small
I\left( \alpha\left(\boldsymbol{\mu}_l^h - \boldsymbol{o}_l^h\right); \boldsymbol{X}[v_s:v_e] \mid \boldsymbol{o}_l^h \right) \geq 0.
\end{equation}
Since \( \boldsymbol{\mu}_l^h \) is constructed by normalizing the attention weights corresponding to the visual tokens \( \boldsymbol{X}[v_s:v_e] \), the normalization factor is clearly greater than 1. As a result, subtracting \( \boldsymbol{o}_l^h \) from \( \boldsymbol{\mu}_l^h \) retains the positive direction corresponding to the visual tokens and removes the non-visual components, which are independent of \( \boldsymbol{X}[v_s:v_e] \).
Thus, the difference \( \boldsymbol{\mu}_l^h - \boldsymbol{o}_l^h \) still carries information about the visual region \( \boldsymbol{X}[v_s:v_e] \) after subtracting the non-visual, independent part. Since \( \alpha \in (0, 1] \) is strictly positive, the conditional MI Eq.~\eqref{miuo} is guaranteed to be non-negative. The equality holds only when no intervention is applied. 
\end{proof}
\begin{table*}[t]
\centering
\caption{Hallucination mitigation results based on LLaVA-1.5-7B evaluated by Accuracy and F1-score on the POPE benchmark across three evaluation settings (\emph{Random}, \emph{Popular}, \emph{Adversarial}) and \emph{Average}. The best and second-best results are in \textbf{bold} and \underline{underlined}. The results in the table are sourced from the corresponding articles and all hallucination experiments on Qwen-VL are in Appendix~\ref{apx:C}.}
\resizebox{1.0\linewidth}{!}{
\begin{tabular}{l l c c c c c c c c}
\toprule
\multirow{2}{*}{Evaluation} 
& \multirow{2}{*}{Methods} 
& \multicolumn{2}{c}{Random}  
& \multicolumn{2}{c}{Popular} 
& \multicolumn{2}{c}{Adversarial} 
& \multicolumn{2}{c}{Average} \\ 
\cmidrule(r){3-4} \cmidrule(lr){5-6} \cmidrule(lr){7-8} \cmidrule(l){9-10}
&  
& Accuracy $\uparrow$ & F1-score $\uparrow$ 
& Accuracy $\uparrow$ & F1-score $\uparrow$  
& Accuracy $\uparrow$ & F1-score $\uparrow$  
& Accuracy $\uparrow$ & F1-score $\uparrow$  \\
\midrule
\multirow{6}{*}{MSCOCO}
& LLaVA-1.5-7B                    & 83.49 & 82.28 & 79.98 & 79.34 & 76.03 & 76.26 & 79.83 & 79.29 \\
& + VCD \citep{vcd2024cvpr}       & 86.84 & 86.83 & 82.65 & 83.37 & 77.31 & 79.28 & 82.27 & 83.16 \\
& + ICD \citep{wang2024mitigating}& 84.87 & 83.27 & 82.93 & 81.45 & 81.07 & 79.96 & 82.96 & 81.56 \\
& + OPERA \citep{opera2024cvpr}   & \underline{87.53} & \underline{86.45} & \underline{84.21} & 83.50 & 80.88 & 80.69 & \underline{84.21} & 83.55 \\
& + INTER~\cite{dong2025inter}    & 86.21 & 85.77 & 83.05 & \underline{84.34} & \underline{81.15} & \textbf{87.02} & 83.47 & \textbf{85.71} \\
\rowcolor{gray!20}
\cellcolor{white}& + V-ITI (ours) & \textbf{89.74} & \textbf{87.72} & \textbf{84.96} & \textbf{84.77} & \textbf{86.31} & \underline{82.44} & \textbf{87.00} & \underline{84.98} \\
\midrule
\multirow{6}{*}{A-OKVQA}
& LLaVA-1.5-7B                    & 83.45 & 82.56 & 79.90 & 79.59 & 74.04 & 75.15 & 79.13 & 79.10 \\
& + VCD \citep{vcd2024cvpr}       & 86.15 & 86.34 & 81.85 & 82.82 & 74.97 & 77.73 & 80.99 & 82.30 \\
& + ICD \citep{wang2024mitigating}& 85.57 & 85.06 & 81.93 & 81.95 & 77.43 & 78.99 & 81.64 & 82.00 \\
& + OPERA \citep{opera2024cvpr}   & 88.27 & \underline{87.54} & 85.17 & \underline{84.74} & 79.37 & \underline{79.97} & 84.27 & \underline{84.08} \\
& + INTER~\cite{dong2025inter}    & \underline{89.65} & 85.54 & \underline{85.44} & 83.98 & \textbf{81.05} & 78.28 & \underline{85.38} & 82.60 \\
\rowcolor{gray!20}
\cellcolor{white}& + V-ITI (ours) & \textbf{91.51} & \textbf{91.04} & \textbf{87.47} & \textbf{86.97} & \underline{80.34} & \textbf{81.12} & \textbf{86.44} & \textbf{86.37} \\
\midrule
\multirow{6}{*}{GQA}
& LLaVA-1.5-7B                    & 83.73 & 82.95 & 78.17 & 78.37 & 75.08 & 76.06 & 78.99 & 79.13 \\
& + VCD \citep{vcd2024cvpr}       & 86.65 & 86.99 & 80.73 & 82.24 & 76.09 & 78.78 & 81.16 & 82.67 \\
& + ICD \citep{wang2024mitigating}& 84.90 & 84.22 & 78.37 & 78.81 & 75.97 & 76.93 & 79.75 & 79.99 \\
& + OPERA \citep{opera2024cvpr}   & \underline{88.27} & \underline{87.52} & \underline{83.07} & 82.93 & \underline{80.77} & \underline{81.05} & \underline{84.04} & \underline{83.83} \\
& + INTER~\cite{dong2025inter}    & 86.27 & 85.63 & \textbf{84.31} & \underline{83.08} & 79.89 & 78.22 & 83.49 & 82.31 \\
\rowcolor{gray!20}
\cellcolor{white}& + V-ITI (ours) & \textbf{89.98} & \textbf{89.70} & 82.21 & \textbf{84.49} & \textbf{82.04} & \textbf{82.65} & \textbf{84.74} & \textbf{85.61} \\
\bottomrule
\end{tabular}
}
\label{tab:POPE}
\end{table*}
\begin{figure}[h]
\begin{center}
\includegraphics[width=0.95\linewidth]{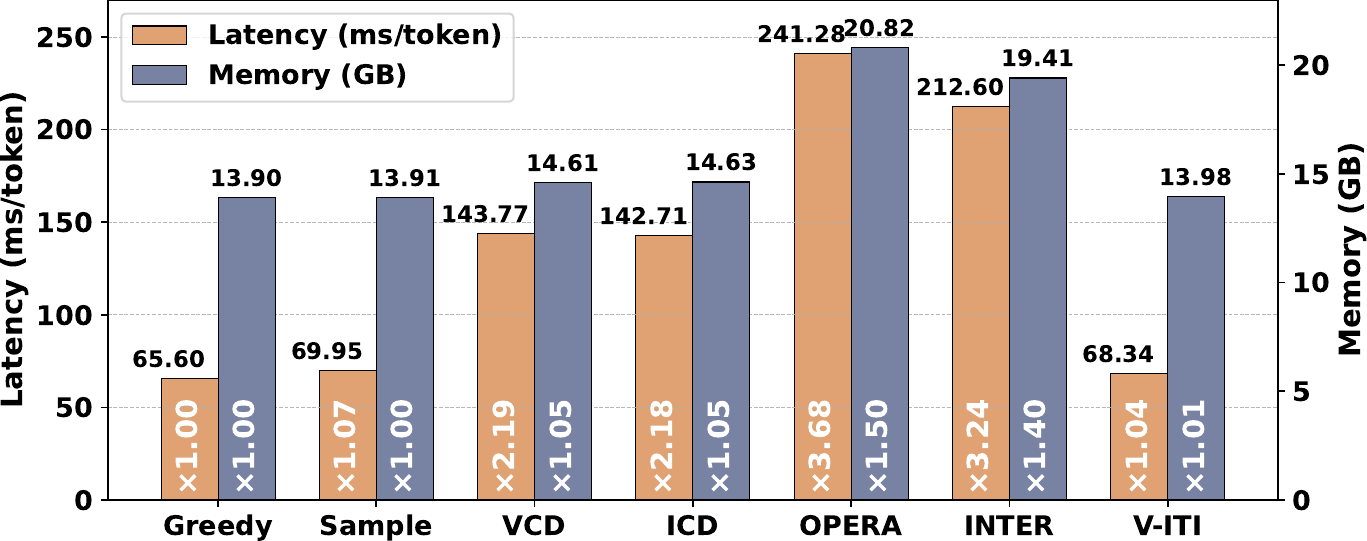}
\end{center} 
\caption{Efficiency comparison of V-ITI against baseline methods. V-ITI achieves near-greedy latency with minimal memory. Logits Intervention (VCD, ICD) average 2.19x latency of Greedy, while Attention Intervention (OPERA, INTER) average 3.46x.}
\label{fig5}
\end{figure}
\noindent\textbf{Time Complexity.} The VND introduces overhead from computing \( p_{\theta_l^h}(\boldsymbol{o}_l^h) \) via the sigmoid dot product, and the VRI from calculating \( \boldsymbol{\mu}_l^h \) (Eq.~\eqref{eq:visual_activation}) and modulated activations \( \boldsymbol{\hat{o}}_l^h \) (Eq.~\eqref{eq:gated_intervention}), both with time complexity \( O(n L H D) \). These linear overheads are negligible compared to the MLLM's self-attention complexity \( O(n^2 L H D) \). As shown in Fig.~\ref{fig5}, V-ITI achieves latency and memory usage almost identical to greedy decoding, with only marginal increases over greedy sampling. In contrast, VCD and ICD are more than 2× slower than V-ITI due to its dual-pass inference, and OPERA and INTER take over 3.2× longer, primarily because of beam search and candidate re-ranking. Notably, our method remains compatible with strategies such as Flash attention, enabling non-disruptive changes.

\begin{table}[t]
\centering
\caption{Performance across methods on the CHAIR benchmark.}
\resizebox{1.0 \linewidth}{!}{
\begin{tabular}{l|ccccc}
\toprule
 Methods  & CHAIR$_S$ $\downarrow$ & CHAIR$_I$ $\downarrow$ & Average $\downarrow$ & Recall $\uparrow$ \\ 
\midrule
LLaVA-1.5  & 50.0 & 15.4 & 32.7 & 77.1 \\
+ VCD      & 48.6 & 14.9 & 31.8 & 77.3 \\
+ ICD      & 56.2 & 16.3 & 36.3 & 16.3 \\
+ OPERA    & 47.8 & 14.6 & 31.2 & 76.8 \\
+ INTER    & 47.0 & 14.2 & 30.6 & 78.6 \\ \midrule
\rowcolor{gray!20}+ V-ITI & 46.1\textcolor{ForestGreen}{\,(+4.1)} & 13.5\textcolor{ForestGreen}{\,(+1.9)} & 29.8\textcolor{ForestGreen}{\,(+3.7)} & 80.4\textcolor{ForestGreen}{\,(+3.3)} \\ 
\bottomrule
\end{tabular}}
\label{tab:CHAIR} 
\end{table}
\section{Experiments}\label{sec:experiments}
This section empirically validates the effectiveness of V-ITI in hallucination mitigation and general-purpose task performance through tests conducted across diverse models, while addressing four key research questions:
\begin{itemize}
\item \textbf{RQ1}: How does V-ITI compare with SOTA methods in hallucination mitigation performance?
\item \textbf{RQ2}: How does V-ITI perform on general vision-language tasks compare with other methods?
\item \textbf{RQ3}: How does V-ITI preventing over-intervention compare with other methods?
\item \textbf{RQ4}: How do different parameter configurations affect V-ITI's overall performance?
\end{itemize}
\begin{table*}[t]
\small
\centering
\caption{General vision-language performance across multiple benchmarks. Higher is better for all metrics. The best and second-best results are in \textbf{bold} and \underline{underlined}.
Following the same setup as in the INTER work, OPERA is not adapted to the Qwen-VL model.}
\begin{tabular}{l c c c c c c}
\toprule
\multirow{2}{*}{Methods} 
& VizWiz-VQA 
& \multicolumn{3}{c}{MME} 
& LLaVA-Wild 
& MM-Vet \\
\cmidrule(lr){2-2} \cmidrule(lr){3-5} \cmidrule(lr){6-6} \cmidrule(lr){7-7}
& Accuracy $\uparrow$ 
& Perception $\uparrow$ 
& Cognition $\uparrow$ 
& Overall $\uparrow$ 
& Average $\uparrow$ 
& Total $\uparrow$ \\
\midrule

LLaVA-1.5-7B                    & 50.00 & 1508.97 & \underline{355.71} & 1864.68  & \underline{64.80} & 31.1 \\
+ VCD \citep{vcd2024cvpr}       & 44.90 & \underline{1515.01} & 357.86 & \underline{1872.87}  & 63.21 & 30.2 \\
+ ICD \citep{wang2024mitigating}& 37.62 & 1306.91 & 287.86 & 1594.77  & 56.90 & 25.9 \\
+ OPERA \citep{opera2024cvpr}   & \underline{50.76} & 1473.62 & 310.71 & 1784.34  & 64.31 & \textbf{32.0} \\
+ INTER~\cite{dong2025inter}    & 48.77 & 1502.35 & 336.18 & 1838.53  & 61.70 & 30.9\\\midrule
\rowcolor{gray!20}
+ V-ITI (ours)                  & \textbf{51.72} & \textbf{1518.32} & \textbf{369.03} & \textbf{1887.35}  & \textbf{65.44} & \underline{31.7} \\
\midrule\textbf{}
Qwen-VL-Chat                    & \underline{66.05} & 1442.79 & 342.14 & 1784.93  & \underline{68.50} & \underline{49.0} \\
+ VCD \citep{vcd2024cvpr}       & 34.54 & 1403.17 & 317.89 & 1721.06  & 53.77 & 34.6 \\
+ ICD \citep{wang2024mitigating}& 29.37 & 1472.06 & \textbf{361.69} & \textbf{1833.75}  & 56.60 & 31.7 \\
+ INTER~\cite{dong2025inter}    & 44.31 & \underline{1451.69} & 345.28 & 1796.97  & 59.84 & 42.3 \\\midrule
\rowcolor{gray!20}
+ V-ITI (ours)                  & \textbf{66.87} & \textbf{1477.14} & \underline{348.28} & \underline{1825.42}  & \textbf{68.84} & \textbf{49.8} \\
\bottomrule
\end{tabular}
\label{tab:general_benchmarks}
\end{table*}


\subsection{Experimental Setup}
\begin{table}[t]
\centering
\caption{Performance on the HallusionBench benchmark.}
\resizebox{1.0 \linewidth}{!}{
\begin{tabular}{l|ccccc}
\toprule
 Methods & fACC $\uparrow$ & qACC $\uparrow$ & ${easy}$A $\uparrow$ & ${hard}$A $\uparrow$ \\ 
\midrule
LLaVA-1.5   & 17.9 & 8.13 & 36.0 & 36.7  \\
+ VCD       & 13.9 & 11.4 & 33.0 & 34.7  \\
+ ICD       & 13.9 & 8.35 & 36.9 & 33.5  \\
+ OPERA     & 16.2 & 5.49 & 37.6 & 35.4  \\
+ INTER     & 15.8 & 8.21 & 36.9 & 34.5  \\ \midrule
\rowcolor{gray!20}+ V-ITI &17.9\textcolor{gray!80}{\,(+0.0)}  &10.27\textcolor{ForestGreen}{\,(+2.14)}  &36.6\textcolor{ForestGreen}{\,(+0.6)}  &37.0\textcolor{ForestGreen}{\,(+0.3)} \\ 
\bottomrule
\end{tabular}}
\label{tab:HallusionBench} 
\end{table}
\textbf{Evaluation Benchmarks.} 
To rigorously assess the effectiveness of our proposed V-ITI, we conduct a comprehensive set of experiments across two categories of benchmarks: three hallucination evaluation benchmarks (CHAIR~\cite{rohrbach2018object}, POPE~\cite{li2023evaluating}, HallusionBench~\cite{guan2024hallusionbench}) and five general vision-language task benchmarks (VizWiz-VQA~\cite{gurari2018vizwiz}, MME~\cite{fu2023mme}, MMBench \cite{liu2023mmbench}, LLaVA-Wild~\cite{liu2024visual}, MM-Vet~\cite{2024MMVet}). More details are in Appendix ~\ref{apx:B.1}.

\noindent\textbf{Baselines.} We use four representative baselines tested on LLaVA-1.5~\cite{liu2024visual} and Qwen-VL~\cite{bai2023qwen} from two mainstream inference-time paradigms for MLLM hallucination mitigation: Logits Intervention (ICD \cite{wang2024mitigating}, VCD \cite{vcd2024cvpr}), and Attention Intervention (OPERA \cite{opera2024cvpr}, INTER~\cite{dong2025inter}). All the methods follow the default configurations from the original papers.  More implementation details are in Appendix~\ref{apx:B.2}.

\subsection{Results on Hallucination Benchmarks (RQ1)}
We evaluate V-ITI across three hallucination benchmarks: CHAIR, POPE, and HallusionBench. V-ITI demonstrates the strongest overall performance across all benchmarks. As shown in Tab.~\ref{tab:CHAIR}, V-ITI reduces the average hallucination score by 3.7 points, from 32.7 to 29.8, representing an 11.3\% relative reduction, while also increasing Recall by 3.3\% compared to LLaVA-1.5. In POPE, presented in Tab.~\ref{tab:POPE}, V-ITI consistently outperforms LLaVA-1.5 across all datasets. Notably, V-ITI achieves the best performance in the average accuracy column on all datasets, with improvements of 7.13, 7.31, and 5.75 points, respectively, over the baseline. Additionally, on A-OKVQA and GQA, V-ITI attains the highest average F1-score, improving by 7.27 and 6.48 points over LLaVA-1.5, despite the competitive baseline. Finally, as shown in Tab.~\ref{tab:HallusionBench}, HallusionBench further validates V-ITI's effectiveness, with improvements of 2.14 in qACC and 0.3 in ${hard}$A, proving robustness against both object-level and reasoning-driven hallucinations.

\begin{figure}[t]
\begin{center}
\includegraphics[width=0.8\linewidth]{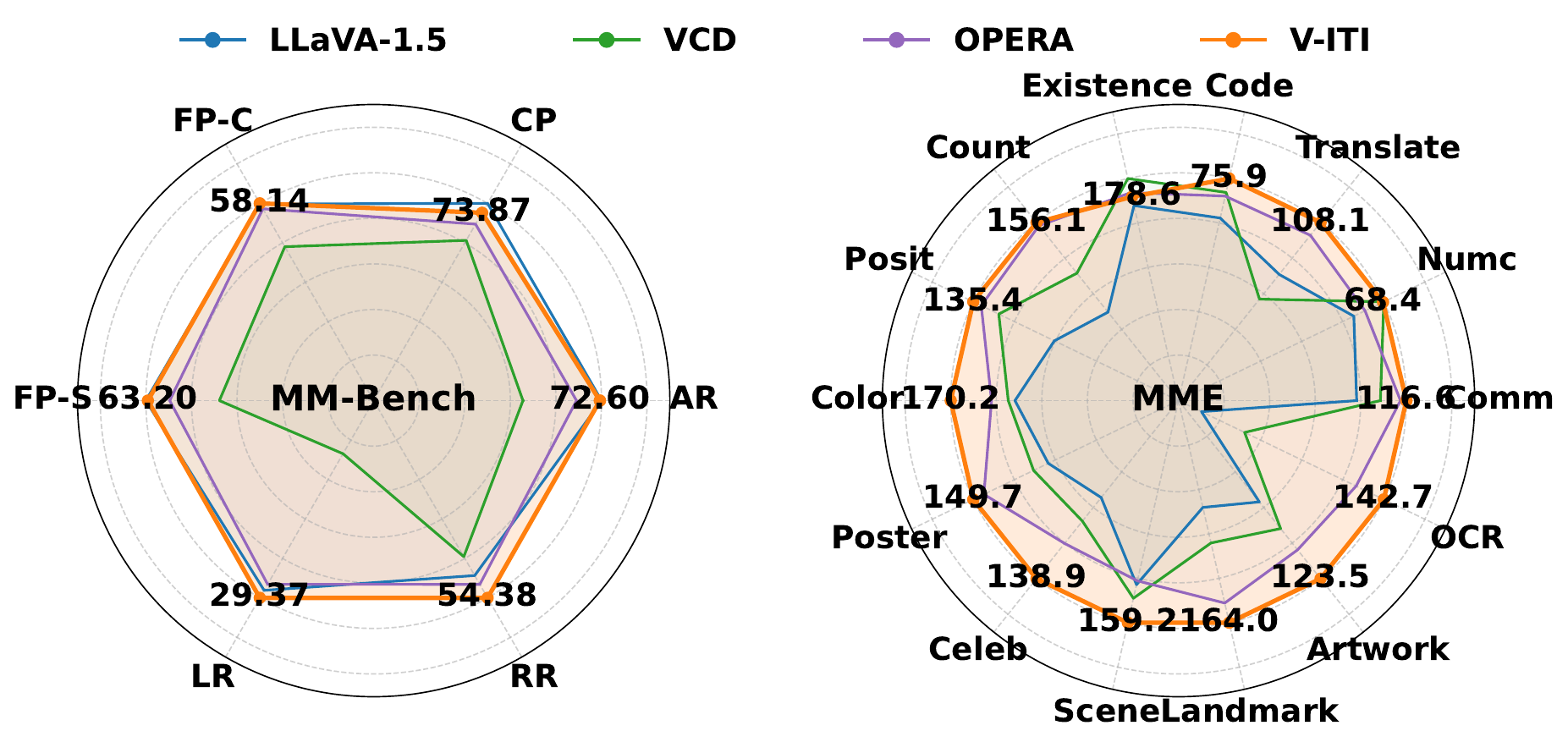}
\end{center} 
\caption{Radar charts on MM-Bench and MME benchmark.}
\label{fig6}
\end{figure}
\begin{figure*}[t]
\begin{center}
\includegraphics[width=0.95\linewidth]{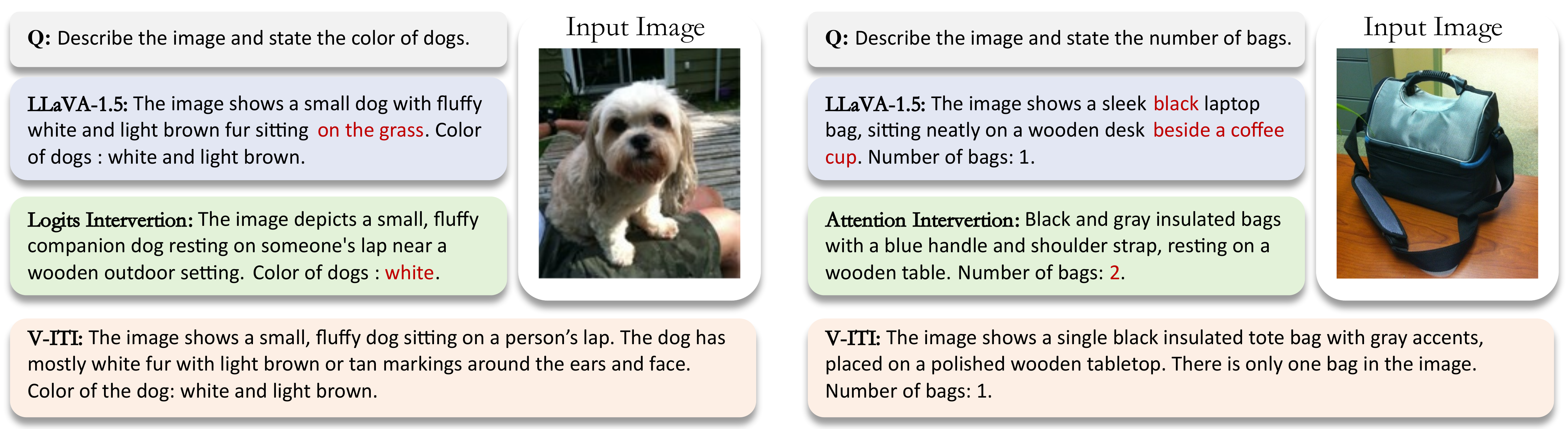}
\end{center} 
\caption{Illustration of over-intervention prevention by our proposed V-ITI with two samples. Hallucinated issues are highlighted in \textcolor{BrickRed}{red}. In the left sample, LLaVA-1.5 hallucinates the dog’s location, and the Logits Intervention method incorrectly alters the dog's color. V-ITI accurately describes the dog’s position and color. In the right sample, LLaVA-1.5 misidentifies the bag's color and adds a non-existent coffee cup, while Attention Intervention overestimates the number of bags. V-ITI avoids these errors and provides a correct description.}
\label{fig7}
\end{figure*}
\subsection{Results on General Tasks (RQ2)}
We evaluate V-ITI across five general-purpose benchmarks: VizWiz-VQA, MME, LLaVA-Wild, and MM-Vet. 
As shown in Tab.~\ref{tab:general_benchmarks}, V-ITI achieves the highest MME Perception score on both LLaVA-1.5 and Qwen-VL, surpassing the state-of-the-art by 9.35 and 25.45 points, respectively. 
Meanwhile, V-ITI delivers the best performance on VizWiz-VQA, improving over LLaVA-1.5 by 3.44\% and over Qwen-VL by 1.24\%. Over two CD-based methods, V-ITI averagely improves by 8.97\% on LLaVA-Wild and 13.01\% on MM-Vet. Notably, the detailed breakdown of MM-Bench and MME in Fig.~\ref{fig6} reveals that these improvements are consistent across all sub-tasks. Specifically, V-ITI excels in Existence, Count, Color, and Scene, showcasing its robustness in diverse visual reasoning tasks. In contrast to many existing methods that prioritize hallucination mitigation at the cost of general capability, V-ITI achieves improvements in both areas simultaneously, underscoring its effectiveness as a lightweight, comprehensive intervention. 
\begin{table}[t]
\centering
\caption{Ablation study on LLaVA-1.5 across four benchmarks. Our full V-ITI outperforms two variants: \textbf{w/o VND} (no Visual Neglect Detector) and \textbf{w/o VRI} (no Visual Recall Intervenor), highlighting the importance of both components.}
\label{tab:ablation}
\resizebox{1.0 \linewidth}{!}{
\begin{tabular}{lcccc}
\toprule
\textbf{Method} & \textbf{POPE}$\uparrow$ & \textbf{CHAIR}$\downarrow$ & \textbf{VizWiz-VQA}$\uparrow$ & \textbf{MME}$\uparrow$ \\
\midrule
LLaVA-1.5         & 79.83 & 32.7 & 50.00 & 1864.48 \\
\midrule
V-ITI (Ours)      & \textbf{87.00} & \textbf{29.8} & \textbf{51.72} & \textbf{1887.35} \\ 
w/o VND       & \underline{84.31} & \underline{30.4} & 49.82 & \underline{1877.52} \\
w/o VRI       & 80.12 & 34.1 & \underline{50.47} & 1852.47 \\ 
\bottomrule
\end{tabular}}
\end{table}

\subsection{Over-Intervention Preventation (RQ3)}
To validate our model's over-intervention prevention, we selected two samples from VizWiz-VQA. As shown in Fig.~\ref{fig7}, for the left sample, LLaVA-1.5 hallucinated the dog as \textit{``on the grass''} but correctly identified its color. Logits Intervention reduced some spatial hallucination but misidentified the dog’s color as \textit{``white''}. In contrast, V-ITI correctly described the dog as sitting on a lap with white and light brown fur. For the right sample, the original model misidentified the bag's color as \textit{``black''} and hallucinated a \textit{``coffee cup''}. The Attention Intervention method over-interpreted the number of bags as \textit{``2''}, while V-ITI accurately recognized the bag as a single one with black and gray accents and described its location correctly. These results demonstrate that V-ITI effectively prevents over-intervention while maintaining accuracy in vision-language tasks.
\begin{figure}[t]
\begin{center}
\includegraphics[width=0.8\linewidth]{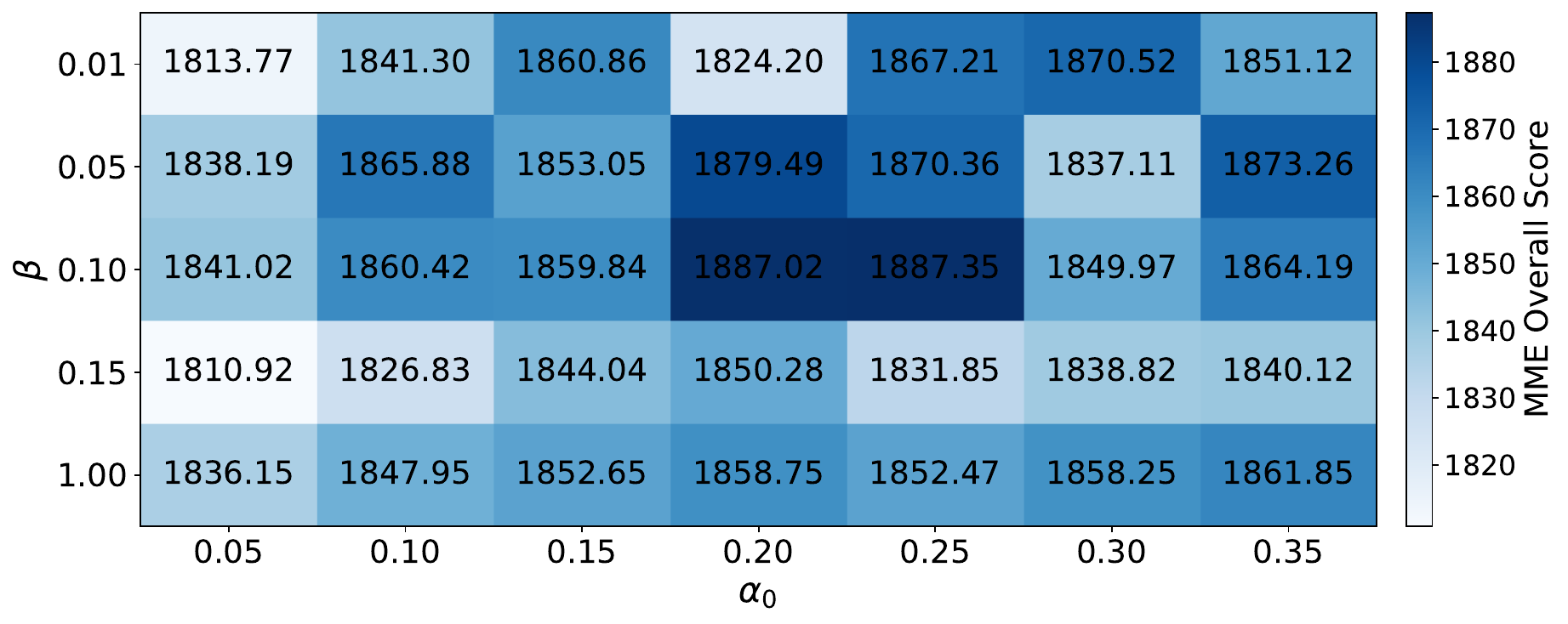}
\end{center} 
\caption{Heatmap of MME scores across different configurations, showing the impact of intervention strength \( \alpha_0 \) and probe selection parameter \( \beta \). The results highlight the importance of selecting an optimal probe set determining when to intervene.
}
\label{fig8}
\end{figure}
\subsection{Ablation Studies (RQ4)}
\noindent\textbf{Component Ablation.}  
We evaluate two key design choices in our framework.  
(i) \textbf{V-ITI w/o VND}: We remove the Visual Neglect Detector and apply intervention unconditionally to all attention heads across the selected layers, mimicking a naive ``always-on'' strategy.  
(ii) \textbf{V-ITI w/o VRI}: We retain the detector but replace our Visual Recall Intervenor with a simplistic intervention  $\mu_l^h = \theta_l^h$ adopted from naive inference-time intervention. Results in Tab.~\ref{tab:ablation} show that both variants underperform the full version. \textbf{V-ITI w/o VND} introduces unnecessary perturbations, causing over-intervention and new hallucinations, while \textbf{V-ITI w/o VRI} fails to recover accurate visual semantics due to the lack of content-aware modulation. These results confirm that \textit{both VND and VRI are essential} for hallucination suppression.

\noindent\textbf{Hyperparameter Sensitivity.}  
We further analyze the impact of intervention strength \( \alpha_0 \), and the probe selection parameter \( \beta \). We vary \( \alpha_0 \) from 0.05 to 0.40 and evaluate \( \beta \) ranging from 1\% to 25\%, along with 50\% and 100\%. Fig.~\ref{fig8} presents a heatmap of the \textbf{Overall} scores on the MME benchmark across these configurations. The results show a peak at \( \alpha_0 \in [0.20, 0.25] \) and \( \beta = 10\% \), indicating balanced intervention most effectively addresses visual neglect.

\section{Conclusion}
In this work, we identify the overlooked issue of over-intervention in hallucination mitigation for MLLMs. To address this, we propose V-ITI, a two-module framework that tackles the critical questions of when and how to intervene. The Visual Neglect Detector determines when intervention is needed via head-level activation paterns, while the Visual Recall Intervenor refines attention toward relevant visual tokens. Extensive experiments demonstrate that V-ITI effectively mitigates hallucinations 
while preserving strong performance on general vision-language tasks.
\vspace{-15pt}
{
    \small
    \bibliographystyle{ieeenat_fullname}
    \bibliography{main}
}
\clearpage
\setcounter{page}{1}
\maketitlesupplementary

\section{Experimental Setup}\label{apx:B}
\subsection{Evaluation Benchmarks}\label{apx:B.1}
In this section, we provide additional details regarding the benchmarks used to comprehensively evaluate the effectiveness of our proposed V-ITI. Below, we outline the three hallucination evaluation benchmarks (CHAIR~\cite{rohrbach2018object}, POPE~\cite{li2023evaluating}, HallusionBench~\cite{guan2024hallusionbench}) and five general vision-language task benchmarks (VizWiz-VQA~\cite{gurari2018vizwiz}, MME~\cite{fu2023mme}, MMBench \cite{liu2023mmbench}, LLaVA-Wild~\cite{liu2024visual}, MM-Vet~\cite{2024MMVet}) included in our experiments.
\begin{table*}[t]
\centering
\caption{Performance of Qwen-VL on POPE. Results are averaged across the MS-COCO, A-OKVQA, and GQA. The best and second-best results are in \textbf{bold} and \underline{underlined}. OPERA relies on the older environment, and experiments with Qwen-VL were not disclosed.}
\begin{tabular}{l c c c c c c c}
\toprule
\multirow{2}{*}{Methods} 
& \multicolumn{2}{c}{Random} 
& \multicolumn{2}{c}{Popular} 
& \multicolumn{2}{c}{Adversarial} \\
\cmidrule(lr){2-3} \cmidrule(lr){4-5} \cmidrule(lr){6-7}
& Accuracy $\uparrow$ & F1-score $\uparrow$ 
& Accuracy $\uparrow$ & F1-score $\uparrow$ 
& Accuracy $\uparrow$ & F1-score $\uparrow$ \\
\midrule
Qwen-VL                          & 88.2 & 87.9 & 82.4 & 83.1 & 77.2 & 78.9 \\
+ VCD~\citep{vcd2024cvpr}        & \underline{89.1} & \textbf{88.4} & 83.0 & 84.1 & \underline{78.8} & \underline{80.1} \\
+ ICD~\citep{wang2024mitigating} & 88.9 & 88.1 & \underline{83.2} & \textbf{84.5} & 78.1 & 79.2 \\
+ INTER~\cite{dong2025inter}     & 87.8 & 87.6 & 82.7 & 82.4 & 76.4 & 77.8 \\\midrule
\rowcolor{gray!20}
+ V-ITI (ours)                   & \textbf{89.8} & \underline{88.2} & \textbf{83.5} & \underline{84.4} & \textbf{80.4} & \textbf{82.1} \\
\bottomrule
\label{tab:POPE_Qwen} 
\end{tabular}
\end{table*}

\begin{table}[!ht]
\centering
\caption{Performance of Qwen-VL on the CHAIR benchmark.}
\resizebox{1.0 \linewidth}{!}{
\begin{tabular}{l|ccccc}
\toprule
 Methods  & CHAIR$_S$ $\downarrow$ & CHAIR$_I$ $\downarrow$ & Average $\downarrow$ & Recall $\uparrow$ \\ 
\midrule
Qwen-VL    & 46.0 & 12.5 & 29.3 & 64.3\\
+ VCD      & 46.8 & 12.3 & 29.6 & 67.9 \\
+ ICD      & 45.0 & 14.3 & 29.7 & 47.6 \\
+ INTER    & 48.7 & 16.4 & 32.6 & 56.2 \\ \midrule
\rowcolor{gray!20}+ V-ITI & 44.2\textcolor{ForestGreen}{\,(+1.8)} & 12.5\textcolor{ForestGreen}{\,(+0.0)} & 28.4\textcolor{ForestGreen}{\,(+0.9)} & 66.4\textcolor{ForestGreen}{\,(+2.1)} \\ 
\bottomrule
\end{tabular}}
\label{tab:CHAIR_Qwen} 
\end{table}

\noindent\textbf{Hallucination Evaluation Benchmarks} includes:
\begin{itemize}
    \item \textbf{CHAIR}~\cite{rohrbach2018object} stands for Caption Hallucination Assessment with Image Relevance, which evaluates the accuracy of image captions by identifying hallucinated objects, i.e., objects that are mentioned in the caption but not present in the image. It includes two versions: CHAIR$_S$ (sentence-level evaluation) and CHAIR$_I$ (instance-level evaluation). For evaluation, we use the \textit{val2014} split of the MSCOCO dataset \cite{lin2014microsoft}, which contains annotations for 80 object categories. A random sample of 500 images is selected from the dataset, and the prompt ``Please describe this image in detail.'' is used to generate captions with a MLLM. The CHAIR metrics are computed by: 
    \[
    \text{CHAIR}_{S} = \frac{ \vert\{ \text{sentences with hallucinated objects}\}\vert }{\vert\{\text{all sentences}\}\vert }
    \]
    \[
    \text{CHAIR}_{I} = \frac{ \vert\{ \text{hallucinated objects}\}\vert }{\vert\{\text{all objects mentioned}\}\vert }
    \]  

    \item \textbf{POPE}~\cite{li2023evaluating} is short for Polling-based Object Probing Evaluation, which is designed to assess object hallucinations in MLLMs. Unlike traditional caption-based approaches, POPE frames hallucination detection as a binary task by posing yes-or-no questions about the presence of specific objects in an image, such as ``Is there a chair in this image?'' To clarify the binary nature of the task, the prompt is followed by ``Please answer yes or no''. The objects referenced in the questions are selected using three sampling methods: random, popular, and adversarial, and performance is evaluated across all of these options. POPE's evaluation is based on four metrics: \textbf{Accuracy}, \textbf{Precision}, \textbf{Recall}, and \textbf{F1 Score}, offering a comprehensive assessment of the model’s ability to detect actual objects and avoid hallucinations.
    \item \textbf{HallusionBench}~\cite{guan2024hallusionbench} is a comprehensive benchmark designed for evaluating image-context reasoning in MLLMs. It presents significant challenges to advanced models, emphasizing nuanced understanding and interpretation of visual data. The benchmark includes 346 images paired with 1129 questions, all crafted by human experts. A novel structure for the visual questions is introduced to establish control groups, enabling a quantitative analysis of the models’ response tendencies, logical consistency, and various failure modes. 
\end{itemize}

\noindent\textbf{General Vision-Language Task Benchmarks} includes:
\begin{itemize}
    \item \textbf{VizWiz-VQA}~\cite{gurari2018vizwiz} is a visual‑question‑answering dataset collected from a natural setting where blind or visually impaired users take mobile phone photos and ask spoken questions about them. The dataset comprises over 31,000 image‑question pairs, each with ten crowdsourced answers. Compared to traditional VQA datasets, the images in VizWiz are often of poorer quality, the questions are more conversational, and there is a substantial portion of unanswerable questions because the required information may be absent from the image. This makes VizWiz a particularly challenging benchmark for VQA models and encourages development of algorithms that can handle real‑world, assistive scenarios.
    \item \textbf{MME}~\cite{fu2023mme} stands for MLLM Evaluation, which is a comprehensive benchmark designed to evaluate the capabilities of MLLMs. The evaluation is divided into two major categories: \textbf{Perception} and \textbf{Cognition}.The Perception category includes fine-grained tasks such as existence, count, position, color, poster identification, celebrity recognition, scene and landmark identification, artwork recognition, and Optical Character Recognition. These tasks focus on the model's ability to understand and process visual information at a detailed level.The Cognition category focuses on tasks that require higher-level reasoning, including commonsense reasoning, numerical calculations, text translation, and code reasoning. These tasks test the model's ability to reason and understand complex relationships beyond just raw visual data. All questions in this benchmark are structured to be answered with a simple ``yes'' or ``no'', providing a clear and straightforward way to measure the model's performance. The benchmark quantifies performance using accuracy and accuracy+, offering an extensive evaluation of the model's perception and cognitive reasoning abilities.
    \item \textbf{MMBench}~\cite{liu2023mmbench} is a bilingual benchmark designed to evaluate MLLMs across a broad spectrum of multimodal skills. It includes over 3,200 multiple‑choice questions covering 20 fine‑grained ability dimensions ranging from perception tasks like object localization and image attribute recognition to reasoning tasks such as logic, relation and social reasoning. To ensure robust evaluation, MMBench introduces the CircularEval strategy, which presents each question multiple times with shuffled answer choices, and uses a large language model to convert free‑form MLLMs outputs into predefined choice labels. The benchmark supports both English and Chinese versions of the questions, enabling direct cross‑language comparison of model performance. 
    \item \textbf{LLaVA-Wild}~\cite{liu2024visual} is a benchmark consisting of 24 diverse images paired with 60 human-crafted questions to evaluate large MLLMs on challenging and varied visual tasks, including conversation, description, and complex reasoning in real-world settings.
    \item \textbf{MM-Vet}~\cite{2024MMVet} short for Multimodal Veterinarian, evaluates MLLMs on integrated vision–language capabilities. It defines six core VL capabilities (recognition, OCR, knowledge, spatial awareness, language generation, math) and examines their 16 possible integrations in complex tasks. MM‑Vet uses an LLM‑based evaluator to score open‑ended outputs across diverse question types with a unified metric.  
\end{itemize}

\begin{table}[!ht]
\begin{flushright}
\caption{Performance of Qwen-VL on the HallusionBench.}
\resizebox{1.0 \linewidth}{!}{
\begin{tabular}{l|ccccc}
\toprule
 Methods & fACC $\uparrow$ & qACC $\uparrow$ & ${easy}$A $\uparrow$ & ${hard}$A $\uparrow$ \\ 
\midrule
Qwen-VL     & 6.65 & 5.93 & 31.43 & 24.88  \\
+ VCD       & 6.88 & 6.12 & 34.28 & 22.34  \\
+ ICD       & 5.98 & 5.49 & 29.87 & 23.65  \\
+ INTER     & 4.73 & 4.98 & 28.85 & 23.26  \\ \midrule
\rowcolor{gray!20}+ V-ITI &7.14\textcolor{ForestGreen}{\,(+0.49)}  &6.05\textcolor{ForestGreen}{\,(+0.12)}  &33.27\textcolor{ForestGreen}{\,(+1.84)}  &25.21\textcolor{ForestGreen}{\,(+0.33)} \\ 
\bottomrule
\end{tabular}}
\label{tab:HallusionBench_Qwen} 
\end{flushright}
\end{table}
\subsection{Baselines and Implementation Details}\label{apx:B.2}
In this section, we provide additional details regarding the baselines and implementation settings used to evaluate the performance of our method. We utilized two mainstream MLLMs, LLaVA-1.5~\cite{liu2024visual} and Qwen-VL~\cite{bai2023qwen}, as the core models. For our baselines, we compare our methods with four representative methods focusing on two types of interventions: Logits Intervention (VCD~\cite{vcd2024cvpr}, ICD~\cite{wang2024mitigating}), and Attention Intervention (OPERA~\cite{opera2024cvpr}, INTER~\cite{dong2025inter}). All configurations for the compared methods adhere to the default settings specified in the original papers.

\noindent\textbf{Logits Intervention} includes:
\begin{itemize}
    \item \textbf{VCD}~\cite{vcd2024cvpr}: Visual Contrastive Decoding is a method designed to mitigate object hallucinations by improving the alignment between the generated captions and the visual content. It operates by leveraging contrastive learning techniques during the decoding process to ensure that the generated output is more grounded in the image's visual features. For VCD tests, we set do\_sample=True, temperature=1, noise\_step=500, plausibility constraint hyperparameter $\lambda$=0.1, and the contrastive emphasis parameter $\alpha$=1, following the default parameter settings from the original code and literature.
    \item \textbf{ICD}~\cite{wang2024mitigating}: Instruction Contrastive Decoding inherits VCD's concept but applies the perturbation in the instruction part instead of the image. It shows that appending role‑prefixes (i.e., “instruction disturbance”) significantly exacerbates hallucinations by increasing multimodal alignment uncertainty. In response, the method contrasts the token probability distributions from the original instruction and the disturbed instruction during inference, thereby identifying and subtracting “hallucinated” concepts. We use the default parameter settings as specified in the original code and literature and set use\_nucleus\_sampling=True, num\_beams=1, top\_p=1.
\end{itemize}

\noindent\textbf{Attention Intervention} includes:
\begin{itemize}
    \item \textbf{OPERA}~\cite{opera2024cvpr} alleviates hallucination in MLLMs by addressing over-trust in the model's predictions. The method applies an over-trust penalty to the attention during beam search, discouraging the generation of overly confident sequences, and introduces a attention retrospection-allocation strategy, which reverts to earlier tokens and reallocates candidates when hallucinations are detected. For OPERA tests, we configured the following parameters: beam=5, sample=True, scale factor=50, threshold=15, num attention candidates=5, and penalty weights=1. 
    \item \textbf{INTER}~\cite{dong2025inter} explicitly guides MLLMs to better utilize their captured multimodal interactions. Specifically, INTER consists of two key modules: the Interactive Guided Locator, which identifies key tokens that significantly influence response accuracy, and the Interaction Probability Modifier, which adjusts the sampling probability of these tokens to strengthen their reliance on multimodal reasoning. We used the beam‑search based INTER experimental settings from the original paper for our evaluations.
\end{itemize}
\noindent\textbf{V-ITI Implementation details.}  All V-ITI tests were conducted using a greedy decoding approach, with do sample=False, temperature=0, beam=1. Additionally, we introduce two more hyperparameters: the intervention strength $\alpha_0$=0.20, which governs the magnitude of the intervention applied during inference, and the probe selection parameter $\beta$=0.10, which regulates the proportion of selected probes. The value of $\alpha_0$ controls the intensity of the intervention, while $\beta$ determines the priority when selecting probes, allowing for more refined control over the intervention process. V-ITI is implemented in the LLaVA-1.5~\cite{liu2024visual} environment, using Python 3.10 and run on eight H100 GPUs.  
\section{Additional Experiments and Results}\label{apx:C} 
In this section, we disclose additional experimental results for the Qwen-VL model, which were not mentioned in the main text. Due to OPERA’s reliance on older versions of Torch and Transformers, it was incompatible with Qwen-VL, and thus experiments with this model were not conducted. The results for Qwen-VL on the POPE, CHAIR and HallusionBench benchmarks are shown in Tab.\ref{tab:POPE_Qwen} Tab.~\ref{tab:CHAIR_Qwen}, and Tab.\ref{tab:HallusionBench_Qwen}. Our method consistently outperforms the Qwen-VL baseline across all benchmarks and metrics while remaining among the top-performing existing approaches, further support the effectiveness of our V-ITI method.


\end{document}